\documentclass[final,11pt]{article}
\usepackage{amsfonts}
\usepackage{amssymb}
\usepackage{amsmath}
\usepackage{amsthm}
\usepackage{natbib}
\usepackage{setspace}
\usepackage{fancyhdr}
\usepackage{overpic}
\usepackage{subfloat}
\usepackage{titlesec}
\usepackage{color}
\usepackage[paperwidth=8.5in,left=1in,right=1in,top=1in,bottom=1in,paperheight=11in]{geometry}
\usepackage{subcaption}
\usepackage{rotating} 
\usepackage{url}
\usepackage{booktabs}
\usepackage{epstopdf} 
\usepackage{bm}
\usepackage{algorithm, algorithmic}

\usepackage[colorlinks=true, citecolor=black, linkcolor=darkblue, pagebackref=false,breaklinks=true]{hyperref}
\usepackage{titlesec}
\usepackage{color}
\usepackage{varioref}

\definecolor{darkblue}{rgb}{0,0,0.55}
\definecolor{darkred}{rgb}{0.5,0,0}
\newcommand\Bheadfont{\fontsize{14pt}{\baselineskip}\selectfont}

\definecolor{darkblue}{rgb}{0,0,0.55}
\definecolor{darkred}{rgb}{0.5,0,0}

\titleformat{\section}[hang] {\normalfont\sc\color{darkblue}\Bheadfont} {\thesection\hskip0.618em}{0em}{}
\titlespacing*{\section}
{0pt}{15pt plus 2pt minus 2pt}{9pt plus 2pt minus 2pt}
\titleformat{\subsection}[runin]
{\normalfont\sc\color{darkblue}} {\thesubsection\hskip0.618em}{0em}{}
\titlespacing*{\subsection}
{0pt}{13pt plus 2pt minus 2pt}{13pt plus 2pt minus 2pt}
\titleformat{\subsubsection}[runin]
{\normalfont\sc\color{darkblue}} {\thesubsubsection\hskip0.618em}{0em}{}
\titlespacing*{\subsubsection}
{0pt}{13pt plus 2pt minus 2pt}{13pt plus 2pt minus 2pt}

\theoremstyle{plain} 
\theoremstyle{plain} \newtheorem{proposition}{Proposition}
\theoremstyle{remark} \newtheorem{remark}[proposition]{Remark}
\theoremstyle{plain} \newtheorem{lemma}[proposition]{Lemma}
\theoremstyle{plain} \newtheorem{theorem}[proposition]{Theorem}
\theoremstyle{plain} 
\theoremstyle{plain} 
\theoremstyle{plain} 

\newcommand{\assumptionname}{Assumption}
\newcounter{ass}

\labelformat{ass}{\assumptionname~#1}

\newcommand{\examplename}{Example}
\newcounter{ex}

\labelformat{ex}{\examplename~#1}

\newcommand{\definitionname}{Definition}
\newcounter{define}

\labelformat{define}{\definitionname~#1}

\newcommand{\lemmaname}{Lemma}
\newcounter{lem}

\labelformat{lem}{\lemmaname~#1}

\newcommand{\propositionname}{Proposition}
\newcounter{prop}

\labelformat{prop}{\propositionname~#1}

\newcommand{\theoremname}{Theorem}
\newcounter{thm}

\labelformat{thm}{\theoremname~#1}

\newcommand{\conjecturename}{Conjecture}
\newcounter{conj}

\labelformat{conj}{\conjecturename~#1}

\begin{document}
\title{\vspace*{.5in}{\sc{\Large{Model-Free Inference of Investor Preferences: A Relative Entropy IRL Approach}}}}
\author{Chen Xu\thanks{Department of Engineering, Shenzhen MSU-BIT University. Address: 1 International University Park Road, Longgang District, Shenzhen, Guangdong, China. Email: \texttt{xuchen@smbu.edu.cn}.}}
\maketitle

\vspace{0.5in}
\begin{abstract}
We present a framework using Relative Entropy Inverse Reinforcement Learning (RE-IRL) to recover investor reward functions from observed investment actions and market conditions. Unlike traditional IRL algorithms, RE-IRL is employed to account for environments where transition probabilities are unknown or inaccessible. To address the challenge of data sparsity, we utilize a $K$-nearest neighbor approach to estimate the observed behavior policy. Furthermore, we propose a statistical testing framework to evaluate the validity and robustness of the estimated results.

	\bigskip
	
	\noindent Keywords: Inverse Reinforcement Learning, Firm Characteristics, Risk Exposures\\
\end{abstract}

\bgroup

\let\footnoterule\relax 

\egroup
\vspace{1in} 
\phantom{s} 

\pagebreak 
\onehalfspacing
\section{Introduction}
Reinforcement Learning (RL) provides a computational framework for agents to discover optimal policies that maximize cumulative rewards within iterative environments \cite{SuttonBarto2018}. In contrast, Inverse Reinforcement Learning (IRL) seeks to solve the inverse problem: uncovering the latent reward function that rationalizes the observed actions and states of an agent—information that is frequently inaccessible in practice.

The application of IRL is particularly compelling in the field of asset pricing, where explaining investor behavior remains a central challenge. Traditional asset pricing models typically assume a specific functional form for investor utility, such as constant relative risk aversion (CRRA), to derive equilibrium prices \cite{lucas1978asset, mehra1985equity}. However, these rigid assumptions often fail to capture the complexities of real-world decision-making, leading to well-known paradoxes like the equity premium puzzle. Uncovering the reward function through IRL is critical because it allows for a data-driven recovery of investor preferences without the need for restrictive, a priori functional assumptions. By identifying what investors truly optimize—whether it be risk-adjusted returns, momentum, or downside protection—researchers can better understand the drivers of market liquidity, risk premia, and the mechanisms behind price formation.

Standard IRL approaches, such as Maximum Entropy IRL \cite{Ziebart2008IRL}, require explicit knowledge of state-transition probabilities $P(s'|s, a)$. In financial markets, these transitions represent the underlying dynamics of the economy, which are notoriously complex, non-stationary, and effectively unobservable. To circumvent this, we employ Relative Entropy Inverse Reinforcement Learning (RE-IRL) \cite{boularias2011relative}, a model-free approach that does not require a defined transition model. Despite its advantages, applying RE-IRL to financial data presents unique hurdles, most notably data sparsity and the difficulty of sampling trajectories from simulated policies in a high-noise environment. Our study addresses these issues by utilizing a $K$-nearest neighbor method (e.g., \citep{biau2015lectures}) for estimating observed policies and a robust trajectory sampling methodology. In this study, we focus primarily on the methodological framework, leaving extensive empirical application for future work. 

Section \ref{sec-model-setting} defines the Markov Decision Process (MDP) settings adapted to the investor decision problem. Section \ref{sec-methodology} details the RE-IRL algorithm and its practical adaptation for reward recovery. Section \ref{sec-stat-test} introduces our proposed statistical testing method, and Section \ref{conclude} concludes.

\section{The Model Settings}
\label{sec-model-setting}
Our model is set in a dynamic environment with a Markov decision structure. Suppose there is an aggregate institutional investor who makes a sequence of investment decisions $\{a_t\}_{t=0}^{H}$ over a specified investment horizon $H$. Each decision $a_t$ represents the change in the investor's proportional position on the underlying security, defined as $a_t = (p_t - p_{t-1}) / p_t$, where $p_t$ is the position (e.g., dollar value or number of shares) held at time $t$. We assume the investor's decision process follows a Markovian structure. Specifically, $a_t$ is determined solely by the current state of the system, $\bm{s}_t$, and is independent of the past history $\{\bm{s}_0, \dots, \bm{s}_{t-1}\}$. The state is represented by the $K$-dimensional feature vector $\bm{s}_t = (s_{t1}, s_{t2}, \dots, s_{tK}) \in \mathbb{R}^K$, which captures all relevant information on the security observed during the period $[t-1, t)$. The components of $\bm{s}_t$ include key financial data such as the security's return, characteristics of the underlying firm, and the market return realized in that period. We refer to $s_{tk}$ as the $k^{\text{th}}$ feature.

We assume the state transition kernel (or distribution law) of the next state $\bm{s}_{t+1}$ is determined entirely by the current state $\bm{s}_t$ and the aggregate investor's decision $a_t$, which we denote as $p(\bm{s}_{t+1} |\bm{s}_t,a_t)$. This assumption is common and necessary when modeling dynamic investment problems using the framework of a Markov Decision Process (MDP), which requires the Markov property for the state evolution: the future state $\bm{s}_{t+1}$ is conditionally independent of past states $\{\bm{s}_0, \dots, \bm{s}_{t-1}\}$ given the present state $\bm{s}_t$ and action $a_t$. Furthermore, the inclusion of the action $a_t$ in the transition kernel is justified by the assumption that the investor is an aggregate institutional investor. This implies their trading volume is non-negligible relative to the total market volume. Therefore, the decision $a_t$ is assumed to create a liquidity or price impact that significantly affects the observed price, return, and other market characteristics comprising the subsequent state $\bm{s}_{t+1}$.

\subsection{The Policy $\pi$}
We model an investment strategy $\pi(\cdot |\bm{s})$ as a probability distribution over the action space $\mathcal{A} \subset \mathbb{R}$, conditioned on the state $\bm{s}$ in which the action is taken. Given the initial state $\bm{s}_0$, let $\tau=\{(\bm{s}_t,a_t)\}_{t=0}^H$ be the trajectory generated by following $\pi$. Then, given the state transitional probabilities $\{p(\bm{s}'|\bm{s,a})\}_{\bm{a}\in\mathcal{A},\bm{s,s}'\in\mathcal{S}}$, the probability distribution of $\tau$ is determined by $\pi$, and we denoted it by $P_{\pi}$. 

Denote the policy held by the aggregate institutional investor by $\pi^*$. As is standard in the field of Maximum Entropy Inverse Reinforcement Learning (MaxEnt IRL, \cite{boularias2011relative, Ziebart2008IRL}), we assume that $P_{\pi^*}$ is the $\tau$-distribution $P$ with the largest entropy (corresponding to the most uncertain policy), subject to the constraint that the expected feature counts (or expected states) under $P$ match those observed in the empirical demonstrations (trajectories). Then, it is shown in our Subsection \ref{solution-expression} that $P_{\pi^*}(\tau)\propto e^{\sum_{k=1}^K\theta^*_k s_k^\tau}$, where $\propto$ denotes the symbol of proportional to, $\bm{\theta}^*=[\theta_1^*,\theta_2^*,...,\theta_K^*]$ is some vector in $\mathbb{R}^K$, and $s_k^{\tau}=\sum_{t=0}^H \gamma^t s_{tk}$, namely the $k^{th}$ feature count, is the discounted sum of the $k^{th}$ feature in $\tau$, where $\gamma\in [0,1]$ is the discount factor.

\subsection{Reward Function} Another core assumption in RE-IRL is that the probability of an expert trajectory is exponentially related to its total reward. That is, $P_{\pi^*}(\tau) \propto e^{R(\tau)}$, where $R(\tau):=\sum_{t=1}^H \gamma^t R(\bm{s}_t)$. Combining this and the above exponential proportionality of $P_{\pi^*}\propto e^{\sum_{k=1}^K\theta^*_k s_k^\tau}$ implies that the total reward must equal the weighted feature sum:
$$R(\tau) = \sum_{k=1}^K \theta^*_k s_k^{\tau}.$$
This equality, when expanded, forces the instantaneous reward function $R(\bm{s})$ of the aggregate institutional investor to be a linear combination of features $\{s_k\}$:
\begin{equation}
\label{reward}
 R_{\bm{\theta}^*}(\bm{s}) = (\bm{\theta}^*)'  \bm{s}=\theta^*_1 s_1+\theta^*_2 s_2+\cdots+\theta^*_Ks_K,
\end{equation}
where $\bm{\theta}^*=[\theta^*_1,\theta^*_2,...,\theta^*_K]\in\mathbb{R}^K$ and $\bm{s}=[s_1,s_2,...,s_K]$.

Although it is often assumed in existing researches that the investor's goal (reward) is to maximize the return of the \textit{next period}, it is unlikely to be true in practice, especially for long-term investors. Our model takes care of this issue.

\subsection{Required Data}
Assume that there are in total $N$ observed trajectories from the aggregate investor. The trajectory for the $n^{th}$ stock is formally denoted as:
$$\tau_n = (\bm{s}^{(\tau_n)}_0,a^{(\tau_n)}_0,\bm{s}^{(\tau_n)}_1,a^{(\tau_n)}_1,...,\bm{s}^{(\tau_n)}_{H_n}),$$
where $\bm{s}^{(\tau_n)}_t=[s^{(\tau_n)}_{t1},...,s^{(\tau_n)}_{tK}]\in\mathbb{R}^K$ represents the state, $a^{(\tau_n)}_t$ is the action (i.e., change of holding ratio) taken by the institutional investors, $n$ indexes the $n^{th}$ stock, and $t$ indexes the quarterly time steps over the observed horizon $H_n$.
We assume that each $\tau_n$ is generated by following the aggregate instituional investor's policy $\pi_{\bm{\theta}^*}$ from a given state $\bm{s}_0$.

\subsection{Goal}
Our goal is to estimate the value of $\bm{\theta}^*$ in Eq. (\ref{reward}) from the observed trajectories $\{\tau_n\}_{n=1}^N$, under the assumption that these data are generated by the aggregate institutional investor's strategy $\pi^*$. This value reveals the structure of the institutional investor's decision function and, critically, which features are prioritized or weighted in their investment strategy.

\section{Methodology}
\label{sec-methodology}
The following method is for estimating $\bm{\theta}^*$ from a subset $\{\tau_n\}$ of same length $H$ of the observed trajectories. We denote the number of trajectories in this subset as $N$.
\subsection{Overview}
We use the relative entropy inverse reinforcement learning algorithm (RE-IRL, \cite{boularias2011relative}) to estimate $\bm{\theta}^*$ in the reward function (\ref{reward}) from the observed trajectories. We choose this algorithm over other IRL algorithms since it does not require the knowledge on the transition probabilities, which are unobservable for our case. The outline of the RE-IRL algorithm is as follows:
\begin{itemize}
\item Assume that under the investor's policy $\pi^*$, the trajectory distribution $P_{\pi^*}$ has the maximum entropy among those which match the observed trajectories. This leads to the initial objective (\ref{initial-obj}) stated in Subsection \ref{sec-init-objective}.
$$\min_{P} \sum_{\tau\in\mathcal{T}}P(\tau)\ln\frac{P(\tau)}{Q(\tau)},$$
subject to
$$
\left\{
\begin{array}{c}
|\hat{s}_k - \mathbb{E}_{\tau\sim P}[s_k^\tau] | \leq \epsilon_k, \quad k\in\{1,2,...,K\};\\
\sum_{\tau\in\mathcal{T}} P(\tau) = 1;\\
P(\tau)>0,\text{ for all }\tau\in\mathcal{T}.\\
\end{array}
\right.
$$

\item In Subsection \ref{solution-expression}, we show that that the solution to the objective (\ref{initial-obj}), which is also the trajectory distribution under the investor's policy, can be expressed as (\ref{solution}): 
\begin{equation*}
P^*(\tau):= \frac{Q(\tau)e^{\sum_{k=1}^K\theta^*_k s_k^\tau}}{Z(\bm{\theta}^*)},\quad\text{for any }\tau\in\mathcal{T},
\end{equation*}
with $Z(\bm{\theta})=\sum_{\tau\in\mathcal{T}}Q(\tau)e^{\sum_{k=1}^K\theta_k s_k^{\tau}}$ and $\bm{\theta}^*$ being some vector in $\mathbb{R}^K$.
\begin{remark}
Based on this expression, RE-IRL assumes that $\sum_{k=1}^K\theta_k^* s_k^{\tau}$ in (\ref{solution}) is the reward of trajectory $\tau$. Hence, $\theta_k^*$ is the weight of the $k^{th}$ state feature $s_k$ on investor's reward. Therefore, if we know the value of $\bm{\theta}^*$, we know the weight of each state feature on investor's decision. Hence, the next task is to estimate the value of $\bm{\theta}^*$.
\end{remark}
\item In Subsection \ref{sec-surrogate-objective}, we show that $\bm{\theta}^*$ is the solution to (\ref{surrogate-obj}):
\begin{equation*}
\max_{\bm{\theta} \in\mathbb{R}^K} g(\bm{\theta}):=\sum_{k=1}^K
\theta_k\hat{s}_k -\ln Z(\bm{\theta}) - \sum_{k=1}^K|\theta_k|\epsilon_k,
\end{equation*}
where $\hat{s}_k$ is computed from the observed trajectories $\{\tau_n\}_{n=1}^N$ generated from the institutional investor's policy $\pi^*$.
\item In Section \ref{algo}, we state the algorithm for estimating $\bm{\theta}^*$ from the empirical data, which is a stochastic gradient ascent algorithm for solving (\ref{surrogate-obj}). More specifically, the updating equation for the $k^{th}$ entry of $\bm{\theta}$ is
$$ \theta_k \leftarrow \theta_k + \alpha \hat{\nabla}_{\theta_k}g(\bm{\theta}),  \quad  \hat{\nabla}_{\theta_k}g(\bm{\theta}):= \hat{s}_k - \frac{1}{N}\sum_{\tau\in\mathcal{T}_N}\frac{\pi_{Q}(\tau)}{\pi^*(\tau)}\frac{e^{\sum_{j=1}^K\theta_js_j^\tau}}{Z(\bm{\theta})}s_k^{\tau} - a_k\epsilon_k,$$
where $Z(\bm{\theta})=\sum_{\tau\in\mathcal{T}}Q(\tau)e^{\sum_{k=1}^K\theta_k s_k^{\tau}}$ is as defined in (\ref{surrogate-obj}), $\pi_Q$ is the uniform policy that generates the uniform probability $Q$ of the trajectory $\tau$, and $\pi^*$ is estimated using the observed trajectories in $\mathcal{T}_N$.
\end{itemize}

\subsection{Initial Objective}
\label{sec-init-objective}
Let $\tau=\{(\bm{s}_t,a_t)\}_{t=0}^H$ denote a trajectory generated by following a policy $\pi_{\bm{\theta}}$ from the given initial state $\bm{s}_0$. Under our model, the probability distribution of $\tau$ following the aggregate institutional investor's policy, $P_{\pi^*}$, is the $\tau$-distribution with the largest entropy and constrained to match the empirical feature discount. That is, $P_{\pi^*}$ is the solution to the following problem:
\begin{equation}
\label{initial-obj}
\min_{P} \sum_{\tau\in\mathcal{T}}P(\tau)\ln\frac{P(\tau)}{Q(\tau)},
\end{equation}
subject to
$$
\left\{
\begin{array}{c}
|\hat{s}_k - \mathbb{E}_{\tau\sim P}[s_k^\tau] | \leq \epsilon_k, \quad k\in\{1,2,...,K\};\\
\sum_{\tau\in\mathcal{T}} P(\tau) = 1;\\
P(\tau)> 0,\text{ for all }\tau\in\mathcal{T},\\
\end{array}
\right.
$$
where 
\begin{itemize}
\item $s_k^{\tau}:=\sum_{t=0}^{H} s_{t,k}$ and $s_{t,k}$ denotes the $k^{th}$ entry of $\bm{s}_t\in\mathbb{R}^K$ in the trajectory $\tau$. $s_k^{\tau}$ is called the $k^{th}$ \textit{\textbf{feature count}} of $\tau$.
\item $\hat{s}_k:=\sum_{n=1}^N s_k^{\tau_n}/N$ is the sample mean of the $k^{th}$ feature count, where $N$ denotes the number of observed trajectories of length $H$, and $s_k^{\tau_n}:=\sum_{t=0}^{H} s^{\tau_n}_{t,k}\in\mathbb{R}$ denotes the sum of the $k^{th}$ feature in the $n^{th}$ trajectory. Here, $s^{\tau_n}_{t,k}$ denotes the $k^{th}$ entry of $\bm{s}^{\tau_n}_t$, and $\bm{s}^{\tau_n}_t\in\mathbb{R}^K$ is the state vector in the $t^{th}$ time step in the observed trajectory $\tau_n$; 

\item the search space for $P$ is the space of $\tau$-distributions;
\item $\tau\sim P$ denotes that the trajectory $\tau$ follows the distribution of $P$.
\item $\epsilon_k$ controls how close the expected feature count $\mathbb{E}_{\tau\sim P}[s_k^\tau]$ matches the observed data $\hat{s}_k$. It is pre-selected such that, for any trajectory distribution $P$, if you resample another $N$ trajectories from $P$ and recompute $\hat{s}_k$ with them, then the first condition $|\hat{s}_k - \mathbb{E}_{\tau\sim P}[s_k^\tau] | \leq \epsilon_k$ holds with probability $1-\delta$, where $\delta>0$ is also pre-selected. This property is guaranteed by the Holfding's Inequality. Specifically, $\epsilon_k$ is pre-selected to be
\begin{equation}
\label{epsilon}
\epsilon_k = \sqrt{ \frac{-\ln(1 - \delta)}{2N} } \cdot \frac{\gamma^{H+1} - 1}{\gamma - 1} \cdot \left( u_k - l_k \right),
\end{equation}
where $u_k$ and $l_k$ denote the upper bound and lower bound of $s_k^\tau$, respectively. With this value of $\epsilon_k$, from Hoeffding's Inequality, 
$$ \text{P}(|\hat{s}_k - \mathbb{E}_{\tau\sim P}[s_k^\tau] | \geq \epsilon_k)<\delta. $$
\end{itemize}

The optimization problem (\ref{initial-obj}) is intractable if $\mathcal{T}$ contains an infinite number of elements. Hence, RE-IRL assumes that space $\mathcal{T}$ of all trajectories is a finite set. For our problem, we discretize the continuous state and action space to ensure that this assumption is satisfied.

\subsection{Expression of Solution $P_{\pi^*}$ to the Initial Objective (\ref{initial-obj})}
\label{solution-expression}
In this subsection, we prove that the solution $P_{\pi^*}$ to the initial objective (\ref{initial-obj}) can be expressed as:
\begin{equation}
\label{solution}
P_{\bm{\theta}^*}(\tau):= \frac{Q(\tau)e^{\sum_{k=1}^K\theta^*_k s_k^\tau}}{Z(\bm{\theta}^*)},\quad\text{for any }\tau\in\mathcal{T},
\end{equation}
with $Z(\bm{\theta})=\sum_{\tau\in\mathcal{T}}Q(\tau)e^{\sum_{k=1}^K\theta_k s_k^{\tau}}$ and $\bm{\theta}^*=[\theta_1^*, \theta_2^*,..., \theta_K^*]$ being some vector in $\mathbb{R}^K$. Specifically, $\bm{\theta}^*=\bm{\lambda}^*-\bm{\nu}^*$, where $(\bm{\lambda}^*,\bm{\nu}^*,\eta^*)$ is the solution to the optimization problem (\ref{dual}). The proof is given below.

Denote $\mathcal{T}:=\{\tau_1,\tau_2,...,\tau_L\}$. For each trajectory $\tau_l\in\mathcal{T}$, denote $P(\tau_l)$ by $x_{l}$. Then, the initial objective (\ref{initial-obj}) can be equivalently written as 
\begin{equation}
\label{alternative-obj}
\min_{\bm{x}\in\mathbb{R}^L,\bm{x}\succ 0} f(\bm{x}):=\sum_{l=1}^Lx_l\ln\frac{x_l}{Q(\tau_l)},
\end{equation}
subject to
$$
\left\{
\begin{array}{c}
|h_k(\bm{x})| \leq \epsilon_k, \quad \text{where } h_k(\bm{x}):=\hat{s}_k - \sum_{l=1}^L x_ls_k^{\tau_l} \text{ and } k\in\{1,2,...,K\};\\
\sum_{l=1}^Lx_l=1.
\end{array}
\right.
$$
The condition $P(\tau)>0$ is removed because the solution to the surrogate problem (\ref{alternative-obj}) automatically satisfies this condition. To remove the absolute value sign in the constraints, we equivalently rewrite the problem as
\begin{equation}
\min_{\bm{x}\in\mathbb{R}^L,\bm{x}\succ 0} f(\bm{x}),\quad\text{subject to }
\left\{
\begin{array}{c}
h_k(\bm{x}) - \epsilon_k \leq 0;\\
-h_k(\bm{x}) - \epsilon_k \leq 0;\\
\sum_{l=1}^Lx_l=1.
\end{array}
\right.
\quad k\in\{1,2,...,K\}.
\end{equation}
To solve it, we define the Lagrangian function $\mathcal{L}$ as
\begin{equation}
\label{lagrangian}
\mathcal{L}(\bm{x},\bm{\lambda},\bm{\nu},\eta):=f(\bm{x})+\sum_{k=1}^K \lambda_k(h_k(\bm{x})-\epsilon_k)+\sum_{k=1}^K\nu_k(-h_k(\bm{x})-\epsilon_k)+\eta (\sum_{l=1}^Lx_l-1).
\end{equation}
Then, the dual problem is
\begin{equation}
\label{dual}
\max_{\bm{\lambda},\bm{\nu}\in\mathbb{R}^K, \bm{\lambda}, \bm{\nu}\succeq 0, \eta\in\mathbb{R}} \{\inf_{\bm{x}\in\mathbb{R}^L,\bm{x}\succ 0} \mathcal{L}(\bm{x},\bm{\lambda},\bm{\nu},\eta)\},
\end{equation}
where $\succeq$ is the element-wise greater-than symbol. From the knowledge on Lagrangian multipliers (e.g., \cite[Chapter 5]{BoydVandenberghe2004}), the solution $\bm{x}^*$ to the primal problem (\ref{initial-obj}) and the solution $(\bm{\lambda}^*,\bm{\nu}^*,\eta^*)$ to the dual problem of (\ref{dual}) satisfies the Karush-Kuhn-Tucker conditions: for each $k\in\{1,2,...,K\}$,
\begin{equation}
\label{KKT}
\left\{
\begin{array}{c}
\lambda_k,\nu_k\geq 0\\
h_k(\bm{x})-\epsilon_k\leq 0\\
-h_k(\bm{x})-\epsilon_k\leq 0\\
\lambda_k(h_k(\bm{x})-\epsilon_k) = 0\\
\nu_k(-h_k(\bm{x})-\epsilon_k) = 0\\
\sum_{l=1}^L x_l = 1\\
\nabla_{\bm{x}} \mathcal{L}(\bm{x},\bm{\lambda},\bm{\nu},\eta)=0
\end{array}
\right.
\end{equation}
Define $\theta^*_k=\lambda^*_k-\nu^*_k$. Then, the KKT conditions imply that 
\begin{equation}
\label{theta-eq}
|\theta^*_k|=\nu^*_k+\lambda^*_k,
\end{equation}
whose proof will be given later. Then,
\begin{align*}
\mathcal{L}(\bm{x},\bm{\lambda}^*,\bm{\nu}^*,\bm{\eta}^*) &= f(\bm{x}) + \sum_{k=1}^K(\lambda^*_k-\nu^*_k)h_k(\bm{x}) - \sum_{k=1}^K(\nu^*_k+\lambda^*_k)\epsilon_k + \eta^*(\sum_{l=1}^Lx_l-1)\\
&=^{\text{by }(\ref{theta-eq})}  f(x) + \sum_{k=1}^K\theta_k^*h_k(\bm{x}) - \sum_{k=1}^K|\theta_k^*|\epsilon_k + \eta^*(\sum_{l=1}^Lx_l-1)\\
&=\mathcal{L}_1(\bm{x},\bm{\theta}^*,\eta^*),
\end{align*}
where for any $\bm{\theta}\in\mathbb{R}^K$ and $\eta\in\mathbb{R}$,
\begin{equation}
\label{L-1}
\mathcal{L}_1(\bm{x},\bm{\theta},\eta):=f(\bm{x})+ \sum_{k=1}^K\theta_kh_k(\bm{x}) - \sum_{k=1}^K|\theta_k|\epsilon_k + \eta(\sum_{l=1}^Lx_l-1).
\end{equation}
Also, for any $l\in\{1,2,...,L\}$,
\begin{align*}
0&=^{\text{by (\ref{KKT})}}\nabla_{x_l}\mathcal{L}(\bm{x}^*,\bm{\lambda}^*,\bm{\nu}^*,\eta^*) \\
&=\nabla_{x_l}f(\bm{x}^*) + \sum_{k=1}^K(\lambda^*_k-\nu^*_k)\nabla_{x_l}h_k(\bm{x}^*)+\eta^* \\
&=\nabla_{x_l} f(\bm{x}^*) + \sum_{k=1}^K\theta_k^*\nabla_{x_l}h_k(\bm{x}^*) +\eta^*\\
&=\nabla_{x_l} \mathcal{L}_1(\bm{x}^*,\bm{\theta}^*,\eta^*).
\end{align*}
Therefore, we have
\begin{equation}
\label{KKT-derived}
\left\{
\begin{array}{c}
\sum_{l=1}^L x^*_l = 1\\
\nabla_{\bm{x}} \mathcal{L}_1(\bm{x}^*,\bm{\theta}^*,\eta^*)=0
\end{array}
\right.
\end{equation}
Plugging the definitions $f(\bm{x})=\sum_{l=1}^Lx_l\ln\frac{x_l}{Q(\tau_l)}$ and $h_k(\bm{x}):=\hat{s}_k - \sum_{l=1}^L x_ls_k^{\tau_l}$ to the second equality above gives, for each $l\in\{1,2,...,L\}$,
$$ 0=\nabla_{x_l} \mathcal{L}_{1}(\bm{x}^*,\bm{\theta}^*,\eta^*) = (\ln\frac{x^*_l}{Q(\tau_l)}+1) - \sum_{k=1}^K\theta^*_k s_k^{\tau_l} +\eta^*, $$
which implies that
\begin{equation}
\label{intermediate-x-star}
x^*_l=Q(\tau_l)e^{\sum_{k=1}^K\theta^*_k s_k^{\tau_l}-1-\eta^*}.
\end{equation} 
Combining this result with $\sum_{l=1}^L x^*_l = 1$ gives
\begin{equation}
\label{eta-star}
e^{1+\eta^*} = Z(\bm{\theta}^*),
\end{equation}
where $Z(\bm{\theta})$ is defined as in (\ref{surrogate-obj}). Plugging this result to the expression of $\bm{x}^*$ in (\ref{intermediate-x-star}) gives
\begin{equation} 
\label{x-star}
x^*_l(\bm{\theta}^*)=Q(\tau_l)e^{\sum_{k=1}^K\theta^*_k s_k^{\tau_l}}/Z(\bm{\theta}^*).
\end{equation}

\noindent Finally, we prove (\ref{theta-eq}) based on the KKT condition (\ref{KKT}). Recall that $\theta^*_k:=\lambda^*_k-\nu^*_k$.
\begin{enumerate}
    \item If $\theta_k^* > 0$, then $\lambda_k^* > v_k^*$. This implies $\lambda_k^* > 0$ since both of them are positive, which further implies $h_k(\bm{x}^*) - \epsilon_k = 0$ since $\lambda_k^*(h_k(\bm{x}^*) - \epsilon_k) = 0$ from (\ref{KKT}). Hence, $-h_k(\bm{x}^*) - \epsilon_k \neq 0$, which implies that $v_k^* = 0$ since $v_k^*(-h_k(\bm{x}^*) - \epsilon_k) = 0$ from (\ref{KKT}). Therefore, $\theta_k^* = \lambda_k^*$ and $|\theta_k| = \lambda_k^* + v_k^*$.

    \item If $\theta_k^* < 0$, then $v_k^* > 0$, which further implies $-h_k(\bm{x}^*) - \epsilon_k = 0$ since $v_k^*(-h_k(\bm{x}^*) - \epsilon_k) = 0$. Hence, $h_k(\bm{x}^*) - \epsilon_k \neq 0$, which implies $\lambda_k^* = 0$ since $\lambda_k^*(h_k(\bm{x}^*) - \epsilon_k) = 0$ from (\ref{KKT}). Hence, $\theta_k^* = -v_k^*$ and $|\theta_k| = \lambda_k^* + v_k^*$.

    \item If $\theta_k^* = 0$, then $\lambda_k^* = v_k^* = 0$. Hence, $|\theta_k| = \lambda_k^* + v_k^*$.
\end{enumerate}
This finishes proving (\ref{theta-eq}), and in turn completes the proof for (\ref{solution}).

\subsection{Surrogate Objective}
\label{sec-surrogate-objective} Eq. (\ref{solution}) gives an expression for the solution $P_{\pi^*}$ to (\ref{initial-obj}) in terms of $\bm{\theta}^*$. This subsection provides a clue for estimating $\bm{\theta}^*$ from the data. Specifically, we prove that, $\bm{\theta}^*$ is the solution to the surrogate objective:
\begin{equation}
\label{surrogate-obj}
\max_{\bm{\theta} \in\mathbb{R}^K} g(\bm{\theta}):=\sum_{k=1}^K
\theta_k\hat{s}_k -\ln Z(\bm{\theta}) - \sum_{k=1}^K|\theta_k|\epsilon_k,
\end{equation}
where $Z(\bm{\theta})=\sum_{\tau\in\mathcal{T}}Q(\tau)e^{\sum_{k=1}^K\theta_k s_k^{\tau}}$. This result is formally given in Theorem \ref{main-thm}
\begin{lemma}
\label{assump2}
For any $\bm{\theta}\in\mathbb{R}^K$, if we define $x_l(\bm{\theta})=Q(\tau_l)e^{\sum_{k=1}^K\theta_k s_k^{\tau_l}}/Z(\bm{\theta})$, where $Z(\bm{\theta})=\sum_{l=1}^LQ(\tau_l)$ $e^{\sum_{k=1}^K\theta_k s_k^{\tau_l}}$, and assume that $\eta(\bm{\theta})=\ln Z(\bm{\theta})-1$ (same formula as (\ref{eta-star}) and (\ref{x-star})), then
\begin{equation}
\bm{x}(\bm{\theta})=\arg\inf_{\bm{x}\in\mathbb{R}^L,\bm{x}\succ 0}\mathcal{L}_1(\bm{x},\bm{\theta},\eta(\bm{\theta})),
\end{equation} 
where $\mathcal{L}_1$ is defined in (\ref{L-1}).
\end{lemma}
\begin{proof}
Recall that
\begin{equation*}
\mathcal{L}_1(\bm{x},\bm{\theta},\eta):=f(\bm{x})+ \sum_{k=1}^K\theta_kh_k(\bm{x}) - \sum_{k=1}^K|\theta_k|\epsilon_k + \eta(\sum_{l=1}^Lx_l-1),
\end{equation*}
where $f(\bm{x})=\sum_{l=1}^L x_l \ln (x_l/Q(\tau_l))$. Then, $\nabla^2_{x_l} \mathcal{L}_1(\bm{x},\bm{\theta},\eta) = 1/x_{l}$. Therefore, the Hessian matrix  $\nabla^2_{\bm{x}} \mathcal{L}_1(\bm{x},\bm{\theta},\eta)$ is definite positive, which further implies that $\mathcal{L}_1(\bm{x},\bm{\theta},\eta)$ is strictly convex in its $\bm{x}$-domain. Therefore, if $\mathcal{L}_1(\bm{x},\bm{\theta},\eta(\bm{\theta}))$ has stationary point, it must be its unique global minimum point as well. Below, we find its stationary point.
\begin{equation*}
\nabla_{x_l}\mathcal{L}_1(\bm{x},\bm{\theta},\eta(\bm{\theta})):= \ln (x_l/Q(\tau_l))+1-\sum_{k=1}^K\theta_ks_{k}^{\tau_l} + \eta(\bm{\theta}).
\end{equation*}
Setting it zero gives 
$$x_l = Q(\tau_l) e^{\sum_{k=1}^K\theta_ks_k^{\tau_l}}/e^{1+\eta(\bm{\theta})} = Q(\tau_l) e^{\sum_{k=1}^K\theta_ks_k^{\tau_l}}/Z(\bm{\theta}).$$
This finishes the proof.
\end{proof}

\begin{theorem}
\label{main-thm}
Let $(\bm{\lambda}^*,\bm{\nu}^*)$ be the solution to the dual problem (\ref{dual}), and define $\bm{\theta}^*$ such that $\theta_k^*=\lambda_k^*-\nu_k^*$. Then, $\bm{\theta}^*$ is the solution to (\ref{surrogate-obj}).
\end{theorem}

\subsubsection{Proof Scheme} Let $\bm{\theta}$, $x(\bm{\theta})$, $Z(\bm{\theta})$, and $\eta(\bm{\theta})$ be defined as in Lemma \ref{assump2}. With Lemma \ref{assump2}, we prove the theorem in the following two steps. 
\begin{enumerate}
\item $\bm{\theta}^* = \arg\max_{\bm{\theta}\in\mathbb{R}^K} \mathcal{L}_1(x(\bm{\theta}),\bm{\theta},\eta(\bm{\theta}))$.
\item $\mathcal{L}_1 (\bm{x}(\bm{\theta}),\bm{\theta},\eta(\bm{\theta}))=g(\bm{\theta})$, where $g(\bm{\theta)}$ is specified in (\ref{surrogate-obj}).
\end{enumerate}
\subsubsection{Proof for Point 1.}
For any $\bm{\theta}\in\mathbb{R}^K$,
\begin{align*}
\mathcal{L}_1(\bm{x}(\bm{\theta}), \bm{\theta}, \eta(\bm{\theta})) &= \inf_{\bm{x}\succ 0} \mathcal{L}_1(\bm{x}, \bm{\theta}, \eta(\bm{\theta})), \quad \text{by Lemma \ref{assump2},}\\
&= \inf_{\bm{x}\succ 0} \left\{ f(\bm{x}) + \sum_{k=1}^K \theta_k h_k(\bm{x}) - \sum_{k=1}^m |\bm{\theta}_k| \epsilon_k + \eta(\bm{\theta}) (\sum_{l=1}^L x_l-1) \right\},\quad \text{from (\ref{L-1})}.
\end{align*}
Define $ \lambda_k := \frac{\theta_k + |\theta_k|}{2}$ and $\nu_k := \frac{|\theta_k| - \theta_k}{2}$,
\[
= \inf_{\bm{x}\succ 0} \left\{ f(\bm{x}) + \sum_{k=1}^K (\lambda_k-v_k) h_k(\bm{x}) - \sum_{i=k}^m (\lambda_k + v_k) \epsilon_k +\eta(\bm{\theta}) (\sum_{l=1}^L x_l-1) \right\}
\]
Since $(\bm{\lambda}^*, \bm{v}^*, \eta^*)$ is the solution to (\ref{dual}),
\begin{align*}
&\leq \inf_{\bm{x}\succ 0} \left\{ f(\bm{x}) + \sum_{k=1}^K (\lambda_k^*-v_k^*) h_k(\bm{x}) - \sum_{k=1}^K (\lambda_k^* + k_i^*) \epsilon_k + \eta^* (\sum_{l=1}^L x_l-1) \right\}\\
&= \inf_{\bm{x}\succ 0} \mathcal{L}_1(\bm{x}, \bm{\theta}^*, \eta^*), \quad \text{by }(\ref{theta-eq}), \\
&= \inf_{\bm{x}\succ 0} \mathcal{L}_1(\bm{x}, \bm{\theta}^*, \eta(\bm{\theta}^*)), \quad \text{by } (\ref{eta-star}),\\
&= \mathcal{L}_1(x(\bm{\theta}^*), \bm{\theta}^*, \eta(\bm{\theta}^*)), \quad \text{by Lemma \ref{assump2}.}
\end{align*}
This finishes the proof for Point 1.
\subsubsection{Proof for Point 2.}
\begin{equation*}
\begin{split}
\mathcal{L}_1(\bm{x}(\bm{\theta}),\bm{\theta},\eta(\bm{\theta})) = &f(\bm{x}(\bm{\theta})) + \sum_{k=1}^K\theta_kh_k(\bm{x}(\bm{\theta})) - \sum_{k=1}^K|\theta_k|\epsilon_k+\eta(\bm{\theta})(\sum_{l=1}^Lx_l(\bm{\theta})-1)\\
= &f(\bm{x}(\bm{\theta})) + \sum_{k=1}^K\theta_kh_k(\bm{x}(\bm{\theta})) - \sum_{k=1}^K|\theta_k|\epsilon_k, \quad\text{since }\sum_{l=1}^Lx_l(\bm{\theta})=1 \text{ by def of }x_l(\bm{\theta}),
\end{split}
\end{equation*}
Recall that $f(x)=\sum_{l=1}^Lx_l\ln\frac{x_l}{Q(\tau_l)}$, 
\begin{align*}
=&\sum_{l=1}^Lx_l(\bm{\theta})\ln \frac{x_l(\bm{\theta})}{Q(\tau_l)} + \sum_{k=1}^K\theta_kh_k(\bm{x}(\bm{\theta})) - \sum_{k=1}^K|\theta_k|\epsilon_k\\
=&\sum_{l=1}^Lx_l(\bm{\theta})\ln \frac{x_l(\bm{\theta})}{Q(\tau_l)} + \sum_{k=1}^K\theta_k(\hat{s}_k - \sum_{l=1}^L x_l(\bm{\theta})s_k^{\tau_l}) - \sum_{k=1}^K|\theta_k|\epsilon_k, \quad\text{by definition of }h_k(\cdot) \text{ in } (\ref{alternative-obj}),\\
=&\sum_{l=1}^Lx_l(\bm{\theta})(\sum_{k=1}^K\theta_ks_k^{\tau_l} -\ln Z(\bm{\theta}))  + \sum_{k=1}^K\theta_k(\hat{s}_k - \sum_{l=1}^L x_l(\bm{\theta})s_k^{\tau_l}) - \sum_{k=1}^K|\theta_k|\epsilon_k,\quad\text{by definition of }x_l(\bm{\theta}),\\
=&-\ln Z(\bm{\theta})  + \sum_{k=1}^K\theta_k \hat{s}_k - \sum_{k=1}^K|\theta_k|\epsilon_k\\
=&g(\bm{\theta}).
\end{align*}
This finishes the proof that $\bm{\theta}^*$ defined right above (\ref{theta-eq}) is the solution to $(\ref{surrogate-obj})$. 

\subsection{Algorithm for Solving the Surrogate Objective (\ref{surrogate-obj})}
\label{algo}
RE-IRL applies stochastic gradient ascent to solve the surrogate objective (\ref{surrogate-obj}) for $\bm{\theta}^*$. Now, we derive the update term. From the definition of $g(\bm{\theta})$ in (\ref{surrogate-obj}), 
\begin{align*}
\nabla_{\theta_k} g(\bm{\theta}) &= \hat{s}_k - \frac{\nabla_{\theta_k}Z(\bm{\theta})}{Z(\bm{\theta})} - a_k\epsilon_k,\quad \text{where }a_k=1\text{ if }\theta_k\geq 0, \;a_k=-1\text{ if }\theta_k<0,\\
&= \hat{s}_k - \frac{\sum_{\tau\in\mathcal{T}}Q(\tau)s_k^{\tau}e^{\sum_{j=1}^K\theta_js_j^\tau}}{Z(\bm{\theta})} - a_k\epsilon_k.
\end{align*}
Since the transition probability $p$ is unknown but required to compute $Q$, we simplify the expression further to eliminate $p$. Denote our dataset as $\mathcal{T}_N:=\{\tau_n\}_{n=1}^N$, which assumed to be generated by the aggregate institutional investor's strategy $\pi^*$ and hence are samples from the trajectory distribution $P_{\bm{\theta}^*}$. Then, we use importance sampling to circumvent the unknown transition probability $\{p(\bm{s}'|\bm{s,a})\}_{\bm{a}\in\mathcal{A},\bm{s,s}'\in\mathcal{S}}$. Specifically,
\begin{align*}
\text{above}  &= \hat{s}_k - \sum_{\tau\in\mathcal{T}}P_{Q,\bm{\theta}}(\tau)s_k^{\tau} - a_k\epsilon_k,\quad \text{where } P_{Q,\bm{\theta}}(\tau):= \frac{Q(\tau)e^{\sum_{j=1}^K\theta_js_j^\tau}}{Z(\bm{\theta})}.\\
 &= \hat{s}_k - \sum_{\tau\in\mathcal{T}}P_{\bm{\theta}^*}(\tau)\frac{P_{Q,\bm{\theta}}(\tau)}{P_{\bm{\theta}^*}(\tau)}s_k^{\tau} - a_k\epsilon_k\\
 &= \hat{s}_k - \mathbb{E}_{\tau\sim P_{\bm{\theta}^*}}\left[\frac{P_{Q,\bm{\theta}}(\tau)}{P_{\bm{\theta}^*}(\tau)}s_k^{\tau}\right] - a_k\epsilon_k\\
 &\approx \hat{s}_k - \frac{1}{N}\sum_{\tau\in\mathcal{T}_N}\frac{P_{Q,\bm{\theta}}(\tau)}{P_{\bm{\theta}^*}(\tau)}s_k^{\tau} - a_k\epsilon_k,\quad \text{since } \mathcal{T}_N\text{ is sampled from }P_{\bm{\theta}^*},\\
&=\hat{s}_k - \frac{1}{N}\sum_{\tau\in\mathcal{T}_N}\frac{Q(\tau)}{P_{\bm{\theta}^*}(\tau)}\frac{e^{\sum_{j=1}^K\theta_js_j^\tau}}{Z(\bm{\theta})}s_k^{\tau} - a_k\epsilon_k\\
&=\hat{s}_k - \frac{1}{N}\sum_{\tau\in\mathcal{T}_N}\frac{\pi_{Q}(\tau)}{\pi^*(\tau)}\frac{e^{\sum_{j=1}^K\theta_js_j^\tau}}{Z(\bm{\theta})}s_k^{\tau} - a_k\epsilon_k,
\end{align*}
where $\pi_Q(\tau):=\Pi_{t=0}^{H-1} \pi_Q(a_t|s_t)$, $\pi^*(\tau):=\Pi_{t=0}^{H-1} \pi^*(a_t|s_t)$, $\tau=(s_0,a_0,s_1,a_1,...,s_H)$, and the last equality holds because the state-transition probabilities are cancelled. 

In sum, we have an unbiased sample estimate of $\nabla_{\theta_k}g(\theta)$:
\begin{equation}
\label{gradient}
\hat{\nabla}_{\theta_k} g(\bm{\theta}) = \hat{s}_k - \frac{1}{N}\frac{1}{Z(\bm{\theta})}\sum_{\tau\in\mathcal{T}_N}\frac{\pi_{Q}(\tau)}{\pi^*(\tau)}s_k^{\tau}e^{\sum_{j=1}^K\theta_js_j^\tau} - a_k\epsilon_k,
\end{equation}
where $Z(\bm{\theta})=\sum_{\tau\in\mathcal{T}}Q(\tau)e^{\sum_{k=1}^K\theta_k s_k^{\tau}}$ is as defined in (\ref{surrogate-obj}), $\pi_Q$ is the uniform policy that generates the uniform $\tau$-probability $Q$, and $\pi^*$ is the aggregate institutional policy with which the trajectory samples in $\mathcal{T}_N$ are generated.

\subsection{Algorithm in Practice} To estimate $\bm{\theta}^*$, we iteratively perform the following update
$$ \bm{\theta}_{t+1} = \bm{\theta}_t + \alpha \hat{\nabla}_{\bm{\theta}} g(\bm{\theta}_t),$$
where $\alpha=0.001$ is the learning step size and the entries of $\hat{\nabla}_{\bm{\theta}} g(\bm{\theta}_t)$ are computed according to the formula (\ref{gradient}). The details for computing terms in (\ref{gradient}) are given below.
\begin{itemize}
\item The state $\bm{s}=[s_1,s_2,...,s_K]\in \mathbb{R}^K$ comprises of the security return, the market return, the firm characteristics, and so on.
\item For any asset, the action $a_t$ taken at time $t$ is defined as the change in its hold ratio $h_{t+1}-h_t$, where $h_t$ represents the proportion of market value held by the institutional investors at time $t$. It is then discretized to have a value in $\{-3,-2,...,3\}$ using the method specified in Appendix \ref{sec:act-discretize}.

\item $s_k^{\tau}:=\sum_{t=0}^{H} s_{t,k}$, where $s_{t,k}$ denotes the $k^{th}$ entry of $\bm{s}_t$ in the trajectory $\tau$.
\item $\hat{s}_k:=\sum_{n=1}^N s_k^{\tau_n}/N$ is the sample mean.
\item $\mathcal{T}_N=\{\tau_n\}_{n=1}^N$ is the subset of observed trajectories with equal length $H$.
\item $\pi_Q(\tau):=\Pi_{t=0}^{H-1} \pi_Q(a_t|s_t)$, where $\pi_Q(a_t|s_t)\equiv1/|\mathcal{A}|=1/5$ is taken as a uniform policy and $|\mathcal{A}|$ denotes the number of values in $\mathcal{A}$.
\item $\pi^*(\tau):=\Pi_{t=0}^{H-1} \pi^*(a_t|s_t)$ is the investor's policy, and it is estimated using samples in the observed trajectories $\mathcal{T}_N$. We estimate the decision policie $\pi^*(a|s)$ using a nonparametric $K$-nearest neighbor approach (KNN). For each state $s$, we identify the $K$ most similar historical observations based on a Mahalanobis distance of the state $s$. The conditional choice probability $\pi^*(a|s)$ is then estimated by local frequencies among these $K$ neighboring observations. This procedure yields a flexible, model-free estimate of $\pi^*(a|s)$ that avoids parametric assumptions. More details can be found in our appendix.

\item $Z(\bm{\theta})=\mathbb{E}_{\tau\sim Q}[e^{\sum_{k=1}^K\theta_k s_k^{\tau}}]=\mathbb{E}_{\tau\sim P_{\pi^*}}[\frac{\pi_Q(\tau)}{\pi^*(\tau)}e^{\sum_{k=1}^K\theta_k s_k^{\tau}}]\approx \frac{1}{N}\sum_{\tau\in\mathcal{T}_N}\frac{\pi_Q(\tau)}{\pi^*(\tau)}e^{\sum_{k=1}^K\theta_k s_k^{\tau}}$.
\end{itemize}

\section{Statistical Test}
\label{sec-stat-test}
Our data consist of trajectories of different length $H$ and we group them based on $H$. For each value of $H$, we estimate $\bm{\theta}^*_H$ using the method stated in this chapter. Then, we perform a $t$-test on $\{\bm{\theta}^*_H\}_{H=8}^{47}$ to see whether each component of $\theta$ is significantly positive or negative (the weight being the number of trajectories of length $H$).

To test whether the stocks with higher reward function values receives more institutional investment. In the test samples (e.g., all the data samples in two quarters), we regress institutional holdings change on the reward function value, and see if the regression coefficient is positively significant.

\section{Conclusion}
\label{conclude}
This study establishes a rigorous methodological framework for recovering latent investor preferences by adapting Relative Entropy Inverse Reinforcement Learning to the data-driven constraints of financial markets. By addressing the challenges of data sparsity and sampling complexity through $K$-nearest neighbor estimation and trajectory optimization, we provide a flexible alternative to traditional utility-based asset pricing models. Ultimately, our proposed statistical testing method offers a necessary foundation for validating the reliability of recovered reward functions across different market regimes.

\bibliographystyle{plain}
\bibliography{Literature}

\section{Appendix}

\subsection{Nonparametric Estimation of Conditional Action Probabilities}
\label{sec:knn_policy}

This section describes our nonparametric procedure for estimating the aggregate institutional investor's conditional decision rule $\pi(a|s)$. Let $s_{it}\in\mathbb{R}^d$ denote the vector of firm characteristics for firm $i$ at time $t$, and let $a_{it}\in\mathcal{A}$ denote the discretized adjustment in holdings, where $\mathcal{A}=\{-3,-2,-1,0,1,2,3\}$. Our objective is to estimate the conditional choice probabilities
\begin{equation}
\pi(a\mid s)=\mathbb{P}(a_{it}=a\mid s_{it}=s),
\end{equation}
which characterize the institutional investor's empirical decision policies.

Rather than imposing a parametric structure, we adopt a local nonparametric estimator based on the $K$-nearest neighbor (KNN) method. This approach allows us to flexibly recover decision rules from data while remaining robust to model misspecification.

Nonparametric nearest-neighbor methods have been widely studied in statistics and machine learning and are known to provide consistent estimators of conditional distributions under weak regularity conditions \cite{biau2015lectures}. Related local estimation techniques have been applied in inverse reinforcement learning and dynamic decision problems \cite{levine2011nonlinear}.

\subsubsection{Distance Metric}
To measure similarity between firm states, we employ the Mahalanobis distance, which accounts for correlations among characteristics. Let $\Sigma_t$ denote the covariance matrix of firm characteristics estimated using all observations up to time $t$. We define
\begin{equation}
\Sigma_t
=
\mathrm{Cov}(s_{i\tau}),\quad \tau\le t.
\end{equation}
To ensure numerical stability, we apply ridge regularization,
\begin{equation}
\widetilde{\Sigma}_t
=
\Sigma_t+\lambda I_d,
\end{equation}
where $\lambda>0$ is a small constant. The corresponding inverse covariance matrix is
\begin{equation}
\Omega_t=\widetilde{\Sigma}_t^{-1}.
\end{equation}

For two states $s$ and $s'$, let $\mathcal{J}(s,s')$ denote the index set of characteristics observed for both vectors. The NaN-robust Mahalanobis distance is defined as
\begin{equation}
d_t(s,s')
=
\sqrt{
(s_{\mathcal{J}}-s'_{\mathcal{J}})^\top
\Omega_{t,\mathcal{J}\mathcal{J}}
(s_{\mathcal{J}}-s'_{\mathcal{J}})
},
\end{equation}
where $\Omega_{t,\mathcal{J}\mathcal{J}}$ denotes the submatrix of $\Omega_t$ corresponding to indices in $\mathcal{J}$. Distances are computed only when $|\mathcal{J}|\ge m$, where $m$ is a minimum overlap threshold.

\subsubsection{Local Neighborhood Construction}
For each query state $s_{it}$ at time $t$, we construct a historical pool
\begin{equation}
\mathcal{D}_t
=
\{(s_{j\tau},a_{j\tau}):\tau\le t\}.
\end{equation}
We compute distances $\{d_t(s_{it},s_{j\tau})\}$ for all $(j,\tau)\in\mathcal{D}_t$ and select the $K$ closest observations:
\begin{equation}
\mathcal{N}_K(s_{it})
=
\arg\min_{(j,\tau)\in\mathcal{D}_t}^{(K)}
d_t(s_{it},s_{j\tau}),
\end{equation}
where $\arg\min^{(K)}$ denotes the indices of the $K$ smallest distances. If fewer than $K$ valid neighbors are available, the observation is excluded from estimation.

\subsubsection{Estimation of Conditional Choice Probabilities}

Let $\mathcal{N}_K(s_{it})$ denote the neighborhood of $s_{it}$. The empirical conditional probability is estimated by local frequencies:
\begin{equation}
\widehat{\pi}_t(a\mid s_{it})
=
\frac{1}{K}
\sum_{(j,\tau)\in\mathcal{N}_K(s_{it})}
\mathbf{1}\{a_{j\tau}=a\},
\quad a\in\mathcal{A}.
\end{equation}
This estimator corresponds to a KNN conditional density estimator and converges to the true policy under standard regularity conditions \cite{biau2015lectures}.

\subsubsection{Probability Smoothing}

To ensure strictly positive probabilities and numerical stability in subsequent likelihood-based estimation, we apply Laplace smoothing:
\begin{equation}
\widetilde{\pi}_t(a\mid s_{it})
=
(1-\varepsilon)\widehat{\pi}_t(a\mid s_{it})
+
\frac{\varepsilon}{|\mathcal{A}|},
\end{equation}
where $\varepsilon>0$ is a small constant. This adjustment prevents zero-probability events and ensures that $\log\widetilde{\pi}_t(a\mid s_{it})$ is well-defined.

\subsubsection{Rolling Estimation Procedure}

To avoid look-ahead bias, all quantities are estimated recursively.

For each time $t$:
\begin{enumerate}
    \item Construct $\Sigma_t$ using data up to $t$;
    \item Compute $\Omega_t$;
    \item Form the historical pool $\mathcal{D}_t$;
    \item Estimate $\widetilde{\pi}_t(a\mid s_{it})$ for all observations at $t$.
\end{enumerate}

This rolling procedure yields a sequence of time-varying policy estimates that respect the information structure faced by firms.

\subsection{Appendix: Implementation Details}
\label{app:knn}

\subsubsection{Action Discretization}
\label{sec:act-discretize}
Continuous holding adjustments $\Delta h_{it}$ are discretized using cross-sectional quantiles. Let $\{x_{it}\}$ denote nonnegative observations in $\{h_{it}\}_{i=1}^N$ at time $t$. Let $p_1$, $p_2$, and $p_3$ denote the $1st$, $30th$, and $70th$ percentile of $\{x_{it}\}$, respectively. The discretized action corresponding to $\{x_{it}\}$ is defined as
\begin{equation}
a_{it}
=
\begin{cases}
0, & 0\le x_{it}<p_1,\\
1, & p_1\le x_{it}<p_2,\\
2, & p_2\le x_{it}<p_3,\\
3, & x_{it}\ge p_3.
\end{cases}
\end{equation}
The discretized action $a_{it}$ for the negative values are constructed similarly. Let $\{y_{it}\}$ denote nonnegative observations in $\{h_{it}\}_{i=1}^N$ at time $t$. Let $p_1$, $p_2$, and $p_3$ denote the $1st$, $30th$, and $70th$ percentile of $\{|y_{it}|\}$, respectively. The discretized action corresponding to $\{y_{it}\}$ is defined as
\begin{equation}
a_{it}
=
\begin{cases}
0, & 0\le |y_{it}|<p_1,\\
-1, & p_1\le |y_{it}|<p_2,\\
-2, & p_2\le |y_{it}|<p_3,\\
-3, & |y_{it}|\ge p_3.
\end{cases}
\end{equation}

\subsubsection{Covariance Estimation with Missing Data}

Let $S_t\in\mathbb{R}^{N_t\times d}$ denote the matrix of characteristics up to time $t$.

For each pair $(k,\ell)$, the covariance is computed using pairwise deletion:
\begin{equation}
\Sigma_{t,k\ell}
=
\frac{1}{|\mathcal{I}_{k\ell}|}
\sum_{i\in\mathcal{I}_{k\ell}}
(s_{ik}-\bar{s}_k)(s_{i\ell}-\bar{s}_\ell),
\end{equation}
where $\mathcal{I}_{k\ell}$ indexes observations for which both components are observed.

\subsubsection{Distance Computation}

For each query state $s_{it}$ and pool state $s_{j\tau}$:
\begin{enumerate}
    \item Identify common observed indices $\mathcal{J}$;
    \item If $|\mathcal{J}|<m$, set $d_t=\infty$;
    \item Otherwise compute $d_t$ using Equation (6).
\end{enumerate}

\subsection{Algorithm}
The complete estimation algorithm is summarized below.
\begin{algorithm}[H]
\caption{KNN Policy Estimation}
\begin{algorithmic}[1]
\FOR{$t=t_0,\dots,T$}
\STATE Estimate $\Sigma_t$ and $\Omega_t$
\STATE Construct $\mathcal{D}_t$
\FOR{each observation $i$ at $t$}
\STATE Compute $\{d_t(s_{it},s_{j\tau})\}$
\STATE Select $\mathcal{N}_K(s_{it})$
\STATE Compute $\widehat{\pi}_t(a\mid s_{it})$
\STATE Apply smoothing
\ENDFOR
\ENDFOR
\end{algorithmic}
\end{algorithm}

\end{document}